\documentclass[sigconf]{arxiv_acmart}
\AtBeginDocument{%
  }

\newtheorem{thm}{Theorem}

\newtheorem{cor}{Corollary}
\newtheorem{rem}{Remark}

\def\E{\mathbb{E}}
\def\R{\mathbb{R}}

\begin{document}

\title{Tight Bounds for Jensen’s Gap with Applications to Variational Inference}


\author{Marcin Mazur}
\orcid{0000-0002-3440-8173}
\affiliation{%
\institution{Jagiellonian University}
\department{Faculty of Mathematics and Computer Science}
\streetaddress{Łojasiewicza 6}
\city{Kraków}
\postcode{30-348}
\country{Poland}}
\email{marcin.mazur@uj.edu.pl}

\author{Tadeusz Dziarmaga}
\orcid{0009-0001-8911-0630}
\affiliation{%
\institution{Jagiellonian University}
\department{Faculty of Mathematics and Computer Science}
\streetaddress{Łojasiewicza 6}
\city{Kraków}
\postcode{30-348}
\country{Poland}}
\email{tadeusz.dziarmaga@student.uj.edu.pl}

\author{Piotr Kościelniak}
\orcid{0000-0003-2389-0643}
\affiliation{%
\institution{Jagiellonian University}
\department{Faculty of Mathematics and Computer Science}
\streetaddress{Łojasiewicza 6}
\city{Kraków}
\postcode{30-348}
\country{Poland}}
\email{piotr.koscielniak@uj.edu.pl}

\author{Łukasz Struski}
\orcid{0000-0003-4006-356X}
\affiliation{%
\institution{Jagiellonian University}
\department{Faculty of Mathematics and Computer Science}
\streetaddress{Łojasiewicza 6}
\city{Kraków}
\postcode{30-348}
\country{Poland}}
\email{lukasz.struski@uj.edu.pl}



\begin{abstract}
  Since its original formulation, Jensen’s inequality has played a fundamental role across mathematics, statistics, and machine learning, with its probabilistic version highlighting the nonnegativity of the so-called Jensen’s gap, i.e., the difference between the expectation of a convex function and the function at the expectation. Of particular importance is the case when the function is logarithmic, as this setting underpins many applications in variational inference, where the term variational gap is often used interchangeably. Recent research has focused on estimating the size of Jensen’s gap and establishing tight lower and upper bounds under various assumptions on the underlying function and distribution, driven by practical challenges such as the intractability of log-likelihood in graphical models like variational autoencoders (VAEs). In this paper, we propose new, general bounds for Jensen’s gap that accommodate a broad range of assumptions on both the function and the random variable, with special attention to exponential and logarithmic cases. We provide both analytical and empirical evidence for the performance of our method. Furthermore, we relate our bounds to the PAC-Bayes framework, providing new insights into generalization performance in probabilistic models.
\end{abstract}




\maketitle

\section{Introduction}

Since its first publication in \citep{jensen1906fonctions}, Jensen's inequality significantly influenced different scientific fields, including pure and applied mathematics, statistical inference, and machine learning. In the probabilistic setting, it proclaims the nonnegativity of the so-called \textit{Jensen's gap}, which (in its general form) 
is given by the following formula:
\begin{equation}\label{eq:jensensgap}
\mathcal{JG}(f,X)=\E\{f(X)\}-f(\E X),
\end{equation}
where $X$ is a random variable and $f$ is a convex function. Probably the most valuable setting is the case when $f=-\log$, due to many important applications of such version of Jensen's inequality in variational inference. Therefore, the term \textit{variational gap} has been frequently used instead. (Nevertheless, our paper refers to the notion of Jensen's gap regardless of the context.)

Recently, the problem of estimating the size of Jensen's gap has been intensively investigated. Most authors were particularly concerned about finding upper or lower bounds for $\mathcal{JG}(f,X)$ under various assumptions on $f$ and/or $X$. Specifically, 
notable bounds were provided for analytic or superquadratic $f$ (see \citep{abramovich2004refining,dragomir2015inequality,walker2014lower,banic2008superquadratic,lovrivcevic2018zipf}), 
and for $X$ that follows a distribution for which (central) moments or more complicated expectations are known
(see \citep{abramovich2004refining,dragomir2001some,dragomir2015inequality,gao2017bounds,liao2019sharpening,pecaric1985companion,walker2014lower}).
On the other hand, many related results obtained for the logarithmic function $f$ were motivated by the necessity of applying variational inference methods in training probabilistic models to deal with the problem of the intractability of the log-likelihood of data. In such situations, optimizing appropriate lower (or occasionally upper) bounds instead occurred as an effective solution. Among others, this was the case of variational autoencoder (VAE) \citep{kingma2013auto,rezende2015variational}, which used the evidence lower bound (ELBO) derived directly from Jensen's inequality and therefore suffered from the presence of variational gap. Consequently, finding and tightening bounds on the log-likelihood of the model distribution (and hence the associated variational gap) became an important issue that was intensively investigated by the machine learning community (see, e.g., \citep{burda2015importance,dieng2017variational,ji2019stochastic,grosse2015sandwiching,maddison2017filtering,masrani2019thermodynamic,nowozin2018debiasing,pmlr-v206-struski23a}). Further details concerning related work may be found in Section~\ref{sec:relwork}.

In this paper, we introduce novel general lower and upper bounds for Jensen's gap $\mathcal{JG}(f,X)$. Involving different sets of assumptions on $f$ and $X$, with particular attention paid to the cases of exponential or logarithmic $f$, we provide both analytical and empirical arguments that our approach is superior to those presented in \citep{pmlr-v206-struski23a,liao2019sharpening,lee2021further} in a number of situations. Specifically, experiments conducted on real-world data demonstrate that our approach may provide tighter bounds than other existing techniques and thus prove to be highly effective for estimating the log-likelihood of variational models. Moreover, we illustrate how these bounds can in principle be integrated into the PAC-Bayes framework, suggesting a path towards refined, data-dependent generalization guarantees for probabilistic models, although a full development of such bounds is left to future work.

\section{Related Work}\label{sec:relwork}

In the general setting, several researchers discussed the problem of Jen\-sen's gap estimation. In particular, a variance-based lower bound for strongly convex functions was proposed in \citep{bakula2016converse}, then improved and extended (to an upper bound) in \citep{liao2019sharpening}, and further in \citep{lee2021further}. Both of the latter results were based on Taylor's expansion and had insightful forms both in terms of function derivatives and distribution moments of second and higher order. However, these approaches often yielded boundary values (i.e., 0 and $\infty$) when used to compute bounds for common continuous (e.g., Gaussian) distributions. (In such cases, the authors suggested using a support partitioning method, as proposed in \citep{walker2014lower}.) Other related results have been developed (see, e.g., \citep{abramovich2004refining,dragomir2001some,dragomir2015inequality,horvath2014refinement,pecaric1985companion,walker2014lower,lovrivcevic2018zipf,banic2008superquadratic}), but they either had a complicated form or required additional assumptions about the analyticity or superquadraticity of the function.

In the machine learning setting, the most interesting bounds on Jensen's gap have been those obtained for logarithmic functions. This was due to the intractability of the log-likelihood of probabilistic models designed using variational inference methods. Training such models as variational autoencoders (VAEs) \citep{kingma2013auto,rezende2015variational} or importance-weighted autoencoders (IWAEs) \citep{burda2015importance} was instead reduced to maximization of the respective lower bounds (i.e, ELBO for VAEs and IW-ELBO for IWAEs). On the other hand, several approaches focused on finding the effective upper bound were also considered. Among others, the following proposals were introduced and investigated: $\chi$ upper bound (CUBO) \citep{dieng2017variational}, evidence upper bound (EUBO) \citep{ji2019stochastic}, and an upper bound variant of the thermodynamic variational objective (TVO) \citep{masrani2019thermodynamic}, which was a generalization of EUBO. However, the approach of our particular interest was proposed in~\citep{pmlr-v206-struski23a}, who studied in detail both the general case and the case of the logarithmic function. As the results presented there became the direct motivation for our work, we refer to them many times throughout the paper.

Another line of work where Jensen's inequality plays a central role is in the PAC-Bayes framework for generalization bounds in probabilistic models \citep{mcallester1999pac,catoni2007pac}. These methods typically rely on a first order application of Jensen's inequality to upper-bound the true risk (e.g., cross-entropy) by an expected empirical loss under a posterior distribution. However, a recent work of \citep{NEURIPS2020_3ac48664} notes the looseness of such bounds and proposes variance-based second order corrections. Building on this insight, our work explores higher order generalizations of Jensen's gap and shows how these can, in principle, be integrated into PAC-Bayes bounds to yield tighter and more informative generalization guarantees.

\section{Theory}\label{sec:theory}

As we have already emphasized, our contribution is supported by rigorous mathematical results. We present them in a consistent manner, starting from a general framework and ending with the case of the logarithmic function and the log-normal distribution, which is particularly important for further applications (see Section \ref{sec:applications}).

\subsection{General Bounds}\label{sec:generalbounds}
In this subsection, we present the main results of the paper in the general case. First, we recall simple bounds for Jensen's gap given in \citep{pmlr-v206-struski23a} that involve low order moments, and then we improve them using higher order expectations.

Let $f\colon (a,b) \to \mathbb{R}$ be a continuous function and $X\colon \Omega\to (a,b)$ be a random variable on a probability space $(\Omega, \Sigma, P)$. In addition, to make Jensen's gap $\mathcal{JG}(f,X)$ (given in \eqref{eq:jensensgap}) well defined, assume that $X$ and $f(X)$ have finite first moments.

\begin{thm}[Derived from \citep{pmlr-v206-struski23a}]\label{thm:simple}
If $f$ is a twice differentiable convex function, then we have the following inequalities: 
\begin{equation}\label{eq:simpleboundsold}
\begin{array}{@{}l @{\;} l @{\;} l}
0 &\leq & \mathcal{JG}(f,X)  \\
& \leq &\E\{Xf'(X)\}-\E X \E\{f'(X)\}
 =  \mathrm{cov}\{X,f'(X)\},  
\end{array}
\end{equation}
provided that appropriate finite expected values exist.
\end{thm}

The proof of Theorem \ref{thm:simple} is based on the first order Taylor's expansions. To improve the bounds given in \eqref{eq:simplebounds} we apply higher order Taylor's expansions, which leads to the following general result.

\begin{thm}\label{thm:improved}
If $f$ is a $2k$-differentiable function satisfying $f^{(2k)}(x)$ $ \geq 0$ for any $x\in (a,b)$, then we have the following inequalities:
\begin{equation}\label{eq:simpleboundsimp1}
\begin{array}{@{}l @{\;} r @{\;} l}
\sum_{i=1}^{2k-1} & \!\! a_{i,0} \!
\!& f^{(i)}  (\E X)\E\{(X-\E X)^i\} \leq  \mathcal{JG}(f,X)\\
 &\leq &  -\sum_{i=1}^{2k-1}a_{i,i} \E\{f^{(i)}(X)(X-\E X)^{i}\},
\end{array}
\end{equation}
and (equivalently)
\begin{equation}\label{eq:simpleboundsimp}
\begin{array}{@{}l @{\;} r @{\;} l}
 & \!\!\sum_{i=1}^{2k-1} &  \sum_{j=0}^i   a_{i,j}   f^{(i)}(\E X)\E X^{i-j}(\E X)^{j} \leq  \mathcal{JG}(f,X)\\
 & \leq & -\sum_{i=1}^{2k-1}\sum_{j=0}^ia_{i,j}\E\{f^{(i)}(X)X^{j}\}(\E X)^{i-j},
\end{array}
\end{equation}
where $a_{i,j}=\frac{(-1)^{j}}{i!}{i\choose j}=\frac{(-1)^{j}}{j!(i-j)!}$, provided that appropriate finite expected values exist.
\end{thm}

\begin{proof} Using Taylor's expansion in $\E X$ and in  arbitrary $x\in (a,b)$, we obtain
\begin{equation}\label{eq:lowertaylorimp}
\begin{array}{@{}l @{\;} l @{\;} l}
f(x)-f(\E X)& \geq & \sum_{i=1}^{2k-1}\frac{1}{i!}f^{(i)}(\E X)(x-\E X)^i,
\end{array}
\end{equation}
and thus
\begin{equation}\label{eq:lowertaylorimp_02}
\begin{array}{@{}l @{\;} l @{\;} l}
f(X)-f(\E X) & \geq & \sum_{i=1}^{2k-1}\frac{1}{i!}f^{(i)}(\E X)(X-\E X)^i. 
\end{array}
\end{equation}
Then, integrating the above inequality over all $\omega\in \Omega$, we have 
\begin{equation}\label{eq:lowertaylorimp_03}
\!\!\!\!\begin{array}{@{}l @{\;} l @{\;} l}
\int_\Omega f(X)dP & - & \int_\Omega f(\E X)dP\\
&\geq & \sum_{i=1}^{2k-1}\frac{1}{i!} f^{(i)}(\E X)\int_\Omega(X-\E X)^idP 
\end{array}
\end{equation}
and consequently the lower bound in (\ref{eq:simpleboundsimp1}). Using Newton's formula we obtain the lower bound in (\ref{eq:simpleboundsimp}).

Now we use Taylor's expansion in $x\in (a,b)$ and we obtain 
\begin{equation}\label{eq:uppertaylorimp}
\begin{array}{@{}l @{\;} l @{\;} l}
f(\E X)-f(x) & \geq & \sum_{i=1}^{2k-1}\frac{1}{i!}f^{(i)}(x)(\E X-x)^i,
\end{array}
\end{equation}
and thus 
\begin{equation}\label{eq:uppertaylorimp_02}
\begin{array}{@{}l @{\;} l @{\;} l}
\!\!f(X)-f(\E X) & \leq & -\sum_{i=1}^{2k-1}\frac{1}{i!}f^{(i)}(X)(\E X-X)^i. 
\end{array}
\end{equation}
Then, by integrating the above inequality over all $\omega\in \Omega$, we have 
\begin{equation}\label{eq:uppertaylorimp_03}
\begin{array}{@{}l @{\;} l @{\;} l}
\int_\Omega  &f & \!\!\!(X)  dP - \int_\Omega f(\E X)\; dP\\ &\leq &-\sum_{i=1}^{2k-1}\frac{1}{i!}(-1)^i\int_\Omega f^{(i)}(X)(X-\E X)^idP,
\end{array}
\end{equation}
which implies the upper bound in (\ref{eq:simpleboundsimp1}). By applying Newton's formula, we obtain the upper bound in (\ref{eq:simpleboundsimp}). It completes the proof.
\end{proof}

An immediate consequence of the above theorem is the following corollary, which gives Jensen's gap bounds 
for zero-mean distributions.

\begin{cor}\label{cor:improvedzero}
Under the assumptions of Theorem \ref{thm:improved}, if $\E X=0$ then we have the following inequalities:
\begin{equation}\label{eq:simpleboundsimp2}
\begin{array}{@{}l @{\;} l @{\;} l}
\sum_{i=1}^{2k-1}a_{i,0}f^{(i)}(0)&\!\E& \!\!\!X^i \leq \mathcal{JG}(f,X)\\
&\leq & -\sum_{i=1}^{2k-1}a_{i,i}\E\{f^{i}(X)X^{i}\}.
\end{array}
\end{equation}
\end{cor}

The following remark demonstrates the application of Theorem~\ref{thm:improved} to superquadratic functions within the context of the findings presented in  \citep{abramovich2004refining} and \citep{abramovich2022new}. Recall that a function $f\colon [0,B) \rightarrow \R$ is called superquadratic if, for all $x \in [0,B)$, there exists a constant $C_f(x) \in \R$ such that the inequality
\begin{equation}\label{eq:superq}
\begin{array}{@{}l @{\;} l @{\;} l}
f(y)-f(x)-C_f(x) (y - x) - f (|y - x|) & \geq & 0
\end{array}
\end{equation}
holds for all $y \in [0,B)$.

\begin{rem}\label{rem:abram} 
Let $f\colon [0,\infty)\to \mathbb{R}$ be a superquadratic function that is continuously differentiable with $f(0)=f'(0)=0$ and four times continuously differentiable on $(0,\infty)$ with $f^{(3)}(x)> 0$ and $f^{(4)}(x)\geq 0$ for any $x\in (0,\infty)$. Then, from Theorem 4 of \citep{abramovich2022new}, there exists an interval $[a,b]$, where $a>0$, such that for any discrete random variable $X$ with  $P(X=x_i)=p_i$, where $a=x_1<\ldots <x_n=b$, we have
\begin{equation}\label{eq:Abram}
\begin{array}{@{}l @{\;} l @{\;} l}
 \E\{f(X)\}&-&f(\E X) -\E\{f(|X-\E X|)\}\\ &\geq &\frac{1}{2}f''(a)\mathrm{var} X -\frac{f(b-a)}{(b-a)^2}\mathrm{var} X. 
\end{array}
\end{equation}
But assuming additionally that $X$ has a right-skewed distribution, i.e.,  $\E\{(X-\E X)^{3}\}\geq 0$, one can easily use a similar reasoning combined with Theorem~\ref{thm:improved} to prove the following inequality:\begin{equation}\label{eq::super}
 \begin{array}{@{}l @{\;} l @{\;} l}
 \E\{f(X)\}&-&f(\E X)  - \E\{f(|X-\E X|)\}\\
 &\geq &\frac{1}{2}f''(\E X)\mathrm{var} X -\frac{f(b-a)}{(b-a)^2}\mathrm{var} X,
 \end{array}
\end{equation}
which gives a stricter estimate than \eqref{eq:Abram}, since $f''$ is an increasing function.
Indeed, from Theorem~\ref{thm:improved} in follows that
\begin{equation}\label{eq::proof1}
 \begin{array}{@{}l @{\;} l @{\;} l}
&\!\!\E&\!\!\!\!\{f(X)\}-f( \E X) \\
&\geq &\frac{1}{2}f''(\E X)\mathrm{var} X +\frac{1}{6}f^{(3)}(\E X) \E\{(X-\E X)^3\}\\
&\geq& \frac{1}{2}f''(\E X)\mathrm{var} X,  
\end{array}
\end{equation}
so we obtain (\ref{eq::super}) by observing that $\E\{f(|X-\E X|)\} \leq \frac{f(b-a)}{(b-a)^2}\mathrm{var} X$ (see the proof of Theorem 4 in \citep{abramovich2022new}). 

On the other hand, note that the case where $f$ satisfies $f^{(3)}(x)\geq 0$ and $f^{(4)}(x)\leq 0$ for $x>0$ (e.g., $f(x)=x^2\log x$ for $x>0$ and $f(0)=0$) is handled by superquadraticity, but not by the techniques used here.
\end{rem}

We conclude this subsection by noting a condition that enhances the precision of the estimates presented in Theorem~\ref{thm:improved}. Specifically, we draw inspiration from Theorem 6 of \citep{pmlr-v206-struski23a} to assert that this is the case, for example, for random variables concentrated around their means. More precisely, we assume that for some small $\varepsilon>0$, the condition $|X - \E X|<\varepsilon$ is satisfied almost surely. However, such an assertion is of limited practical utility, as we discuss at the conclusion of the following subsection.


\subsection{Bounds for Logarithmic Function}\label{subsec:boundslog}

In this subsection, we present the results of the paper in the case of $f=-\log$. We retain the structure of Subsection \ref{sec:generalbounds}, giving successively the counterparts of Theorems \ref{thm:simple} and \ref{thm:improved} (note that Corollary~\ref{cor:improvedzero} can no longer be retained).

Let $X\colon \Omega\to (0,\infty)$ be a random variable on a probability space $(\Omega, \Sigma, P)$ such that $X$ and $\log(X)$ have finite first moments.

The following theorem is a direct consequence of Theorem~\ref{thm:simple}, applied to $f=-\log$.

\begin{thm}[Derived from \citep{pmlr-v206-struski23a}]\label{thm:simplelog}
We have the following inequalities: 
\begin{equation}\label{eq:simplebounds}
\begin{array}{@{}l @{\;} l @{\;} l}
\!0 & \leq & \mathcal{JG}(-\log,X)\leq \E X\E X^{-1}-1=\mathrm{cov}(X,X^{-1}),  
\end{array}
\end{equation}
provided that appropriate finite expected values exist.
\end{thm}

To derive (from Theorem \ref{thm:improved}) 
pleasant forms of improved bounds for the logarithmic function $f$, non-obvious combinatorial calculus must be employed.

\begin{thm}\label{thm:improvedlog}
For any positive integer $k$, we have the following inequalities:
\begin{equation}\label{eq:improveboundslog1}
\begin{array}{@{}l @{\;} l @{\;} l}
\sum_{i=1}^{2k-1}\frac{(-1)^i}{i}  (&\!\E&\!\!\!X)^{-i}\E\{(X-\E X)^i\} \leq   \mathcal{JG}(-\log,X) \\
&\leq & -\sum_{i=1}^{2k-1}\frac{1}{i} \E\{X^{-i}(X-\E X)^{i}\},
\end{array}
\end{equation}
and (equivalently)
\begin{equation}\label{eq:improvedboundslog}
\begin{array}{@{}l @{\;} l @{\;} l}
\sum_{j=1}^{2k-1}\frac{1}{j} &+& b_{k,j}\E X^{j}(\E X)^{-j} \leq \mathcal{JG}(-\log,X) \\
&\leq & -\sum_{j=1}^{2k-1}\frac{1}{j}+b_{k,j}\E X^{-j}(\E X)^{j},
\end{array}
\end{equation}
where $b_{k,j}=\frac{(-1)^{j}}{j}{2k-1 \choose j}$, provided that appropriate finite expected values exist.
\end{thm}

\begin{proof} To obtain \eqref{eq:improveboundslog1}, it is enough to use \eqref{eq:simpleboundsimp1} for $f=-\log$ (note that then $f^{(i)}(x)=(-1)^i(i-1)!x^{-i}$). To prove \eqref{eq:improvedboundslog}, we apply \eqref{eq:simpleboundsimp} and appropriate combinatorial calculus. Below, we detail the right-hand side inequality; the other follows analogously.

First, note that using \eqref{eq:simpleboundsimp} for $f=-\log$ leads to the following estimate: 
\begin{equation}\label{eq:simpleboundsimp_log}
\begin{array}{@{}l @{\;} l @{\;} l}
\log&\!\!(&\!\!\!\!\!\!\E X)  -  \E\{\log(X)\}\\
 &\leq & \! -\sum_{i=1}^{2k-1}\sum_{j=0}^i\frac{(-1)^{i+j}}{i}{i \choose j}\E X^{j-i}(\E X)^{i-j},
\end{array}
\end{equation}
which can be rewritten as 
\begin{equation}\label{eq:simpleboundsimp_02}
\!\begin{array}{@{}l @{\;} l @{\;} l}
\log(\E X) & - & \E\{\log(X)\}\\
&\leq & -\sum_{i=1}^{2k-1}\sum_{j=0}^i\frac{(-1)^{j}}{i}{i \choose j}\E X^{-j}(\E X)^{j},
\end{array}
\end{equation}
and, consequently, as
\begin{equation}\label{eq:simpleboundsimp_03}
\begin{array}{@{}l @{\;} l @{\;} l}
&\!\!\log&\!\!(\E X) -  \E\{\log(X)\}\\
&\leq &\!-\sum_{i=1}^{2k-1}\frac{1}{i}-\sum_{j=1}^{2k-1}(-1)^{j}\E X^{-j}(\E X)^{j}\sum_{i=j}^{2k-1}\frac{1}{i}{i \choose j}.
\end{array}
\end{equation}
Then, using the formula $\sum_{i=j}^n \frac{1}{i}{i\choose j}= \frac{1}{j}{n \choose j}$, we obtain the following inequality:
\begin{equation}\label{eq:simpleboundsimp_04}
\begin{array}{@{}l @{\;} l @{\;} l}
\log(\E &\!\!X&\!\!\!)  -  \E\{\log(X)\} \\
&\leq & - \sum_{j=1}^{2k-1}\frac{1}{j}+\frac{(-1)^{j}}{j}{2k-1 \choose j}\E X^{-j}(\E X)^{j},  
\end{array}
\end{equation}
which ends the proof.
\end{proof} 

Note that rescaling the random variable $X$ by a positive constant $a$ (i.e., $X\to aX$) does not change the lower and upper bounds in either \eqref{eq:improveboundslog1} or \eqref{eq:improvedboundslog}.

In the following paragraphs, we examine two essential examples of gamma-distributed and log-normally-distributed random variables. The first case refers to an illustrative example studied in detail in \citep{pmlr-v206-struski23a}, while the other (more important) case concerning log-normally distributed variables can be considered as a basis for our experiments on real-world data (see Subsection \ref{subsec:experiments}), and is therefore followed by necessary additional theoretical outcomes. For a comparison of these results with those that can be obtained using the methods of \citep{liao2019sharpening,lee2021further}, \citep{NEURIPS2020_3ac48664}, and \citep{pmlr-v206-struski23a}, see Section \ref{sec:applications}.

\paragraph{Bounds for Gamma Distributed Random Variable} Consider the random variable $X\sim \text{Gamma}(a,\theta)$, where $a>0$ and $\theta>0$ are shape and scale parameters. Then  $X^{-1}\sim \text{Inv-Gamma}(a,1/\theta)$ and consequently $\E X^j=\theta^j\frac{\Gamma(a+j)}{\Gamma(a)}$ for any integer $j>-a$. Thus, from \eqref{eq:improvedboundslog} and by simple calculation, we obtain the respective bounds for Jensen's gap:
\begin{equation}\label{eq:simpleboundsimpgamma}
\begin{array}{@{}l @{\;} l @{\;} l}
\sum_{j=1}^{2k-1}\frac{1}{j} & + & b_{k,j}\frac{\Gamma(a+j)}{\Gamma(a) a^j}\leq \mathcal{JG}(-\log,X)\\
&\leq & -\sum_{j=1}^{2k-1}\frac{1}{j}+b_{k,j}\frac{\Gamma(a-j )a^j}{\Gamma(a)},  
\end{array}
\end{equation}
provided (only for the upper bound) that $a> 2k-1$.

\paragraph{Bounds for Log-normally Distributed Random Variable} Consider the random variable $X\sim \text{Lognormal}(\mu,\sigma)$ (which means that $\log(X)\sim \text{Normal}(\mu,\sigma)$). Then $\E X^j=\exp(j\mu+\frac{1}{2}j^2\sigma^2)$ for any integer $j$. Thus, we can directly calculate both the exact size of Jensen's gap (see also Theorem 7 in \citep{pmlr-v206-struski23a}):
\begin{equation}\label{eq:jensensgaplognormal}
\begin{array}{@{}l @{\;} l @{\;} l}
\mathcal{JG}(-\log,X) & = & \log(\E X)-\E\{\log(X)\} \\
&= & -\mu+\mu+\frac{1}{2}\sigma^2=\frac{1}{2}\sigma^2, 
\end{array}
\end{equation}
and, applying \eqref{eq:improvedboundslog}, the respective estimates:
\begin{equation}\label{eq:improvedboundslog_for_lognormal}
\begin{array}{@{}l @{\;} l @{\;} l}
\sum_{j=1}^{2k-1}\frac{1}{j} & + & b_{k,j}\exp\{\frac{1}{2}(j^2 -  j)\sigma^2\} \leq \mathcal{JG}(-\log,X) \\
&\leq & -\sum_{j=1}^{2k-1}\frac{1}{j}+b_{k,j}\exp\{\frac{1}{2}(j^2+j)\sigma^2\}.
\end{array}
\end{equation}

It is easy to see that if we have a sequence of random variables $X_n\sim \text{Lognormal}(\mu,\sigma_n)$ with variances tending to $0$ (which implies that $\sigma_n\to 0$), then both Jensen's gap $\mathcal{JG}(-\log,X_n)$ and the respective lower and upper bounds (computed by \eqref{eq:improvedboundslog_for_lognormal} for any $k$) tend to 0. Indeed, the Euler integral representation of the harmonic number $H_n=\sum_{j=1}^n \frac{1}{j}$, given by the following formula (see~\citep{sandifer2006euler}):
\begin{equation}
\begin{array}{@{}l @{\;} l @{\;} l}
    H_n=\int_0^1 \frac{1-x^n}{1-x} dx,  
\end{array}
\end{equation}
allows us to calculate that
\begin{equation}
\begin{array}{@{}l @{\;} l @{\;} l}
    \sum_{j=1}^{2k-1}&\frac{1}{j} &= \int_0^1 \frac{1-x^{2k-1}}{1-x} dx=\int_{0}^1 \frac{1-(1-u)^{2k-1}}{u} du\\
&=&\int_0^1\sum_{j=1}^{2k-1}{2k-1 \choose j}(-u)^{j-1}du\\
&=& \sum_{j=1}^{2k-1}{2k-1 \choose j}\frac{(-1)^{j-1}}{j}=-\sum_{j=1}^{2k-1}b_{k,j},  
\end{array}
\end{equation}
which leads to the conclusion.

We would like to highlight that the aforementioned property is particularly important for our experiments, as it motivates the development of the method of rigorous estimation of the log-likelihood presented in Subsection \ref{subsec:loglikelihoodestimations}.

\paragraph{Bounds Tightening} Theorem \ref{thm:improvedlog} and the importance sampling technique (see \citep{burda2015importance,pmlr-v206-struski23a}) allow us to obtain tight lower and upper bounds on $\log(\E X)$ when $X$ is a positive random variable. To prove this, we rewrite the estimates from \eqref{eq:improvedboundslog} in the following form:
\begin{equation}\label{eq:improveadditivedboundslog}
\begin{array}{@{}l @{\;} l @{\;} l} \E\{\sum_{j=1}^{2k-1}\frac{1}{j} & + & b_{k,j}\frac{X^j}{(\E X)^j}\} \leq  
 \log(\E X)-\E\{\log(X)\}\\
& \leq & \E\{-\sum_{j=1}^{2k-1}\frac{1}{j}-b_{k,j}\frac{(\E X)^j}{X^j}\}.
\end{array}
\end{equation}
Then we replace random variable $X$ with its $n$-sample mean, i.e., $X \to \overline X$ (note that we therefore need to copy $X$ independently $n$ times) and  add $\E\{\log(X)\}$ to all sides. 
Since $\E X=\E \overline X$, this leads to the following estimates:
\begin{equation}\label{eq:improveadditivedboundslogmeans}
\begin{array}{@{}l @{\;} l @{\;} l}
\E&\!\!\{&\!\!\!\!\!\sum_{j=1}^{2k-1}\frac{1}{j}  +   b_{k,j}\frac{\overline X^j}{(\E X)^j} + 
 \log(\overline X)\} \leq  \log(\E X)\\
&  \leq & \E\{-\sum_{j=1}^{2k-1}\frac{1}{j}-b_{k,j}\frac{(\E \overline X )^j}{\overline{X}^j}+\log(\overline X)\}.  
\end{array}
\end{equation}

In \citep{mouri2013log} it is shown that the sum of positive random variables converges to a log-normal distribution (even faster than to any Gaussian distribution, which is ensured by the central limit theorem) as their total number tends to infinity. On the other hand, the variance of the $n$-sample mean converges to 0 as $n\to \infty$. Thus, if $n$ is large, we can treat the estimates given by \eqref{eq:improveadditivedboundslogmeans} as bounds for the log-normal random variable $\overline X$, which are known to be tight (see the previous paragraph), giving the assertion. However, we do not provide theoretical guarantees of convergence. Instead, we provide appropriate empirical arguments (see the ablation study provided in Subsection~\ref{subsec:experiments}).

We emphasize that the above result is not based on the assumption that the support of $X$ lies in a compact interval contained in $(0,\infty)$. (This is particularly important because in practical applications of real-world data, the underlying distributions do not have this property.) Moreover, our experiments confirm that it allows us to obtain superior bounds in a number of situations. Thus, it can be considered a significant improvement over Theorem 5 of \citep{pmlr-v206-struski23a}.

\section{Applications}\label{sec:applications}

To prove the applicability of the theoretical results presented in Section \ref{sec:theory}, we examine them with respect to state-of-the-art solutions. We begin with the toy examples provided in \citep{liao2019sharpening} and \citep{lee2021further} for the exponential function and common continuous distributions (i.e., exponential and Gaussian), and then compare our method with those presented in \citep{NEURIPS2020_3ac48664} and \citep{pmlr-v206-struski23a}. In particular, we demonstrate that our method extends the second-order Jensen's gap approximation proposed in \citep{NEURIPS2020_3ac48664}. Moreover, we conduct experiments on real-world data using an experimental benchmark from \cite{pmlr-v206-struski23a}, which allow us to validate our approach as a novel method for estimating the log-likelihood of probabilistic models. The source code is publicly available at \url{https://github.com/gmum/variational-jensen}.

\subsection{Examples for Exponential Function}\label{subsec:expexp}

In this subsection, we relate our results to those of two recent papers on sharpening Jensen's inequality, i.e., \citep{liao2019sharpening,lee2021further}. Although these works also provide lower and upper bounds for Jensen's gap based on Taylor's expansion, their applicability is influenced by the fact that they often yield boundary values (i.e., 0 and $\infty$) when used to compute bounds for common continuous (e.g., Gaussian) distributions, which cannot happen when our method is applied. Furthermore, these approaches rely on the computation of high order central moments, whereas our bounds are expressed in terms of both central and raw moments, thereby rendering our technique more accessible. (We emphasize that while calculating raw moments from central moments is always possible, the process becomes increasingly complex and tedious as the order increases.).

In the following two paragraphs, we refer to the examples presented in  \citep{liao2019sharpening,lee2021further} for exponential function $f$.

\paragraph{Bounds for Exponentially Distributed Random Variable}

Let $f(x)=\exp(\frac{1}{2}x)$ and $X\sim \text{Exponential}(1)$. The exact size of Jensen's gap is then calculated as follows:
\begin{equation}\label{eq:jensensgapexp}
\begin{array}{@{}l @{\;} l @{\;} l}
\mathcal{JG}(f,X) & = & 2-\sqrt{e}\approx 0.351.
\end{array}
\end{equation}
On the other hand, we have the following estimates:
\begin{equation}\label{eq:jensensgapexpest}
\begin{array}{@{}l @{\;} l @{\;} l}
\sqrt{e}\sum_{i=1}^{2k-1} & \!\!\frac{1}{2^i} & \!\sum_{j=0}^i\frac{(-1)^{i-j}}{(i-j)!}
\leq  \mathcal{JG}(f,X) \\
&\leq & \sum_{i=1}^{2k-1}\sum_{j=0}^i\frac{1}{2^{j-1}}\frac{(-1)^{i-j+1}}{j!},
\end{array}
\end{equation}
which can be obtained by using \eqref{eq:simpleboundsimp} and some simple recalculations. Indeed, note that
\begin{equation}
\begin{array}{@{}l @{\;} l @{\;} l}
    f^{(i)}(X) =(\frac{1}{2})^i\exp(\frac{1}{2}X)
\end{array}
\end{equation}
and therefore
\begin{equation}
\begin{array}{@{}l @{\;} l @{\;} l}
    \E\{f^{(i)}(X)&\!\!X&\!\!\!^{i-j}\} = \E\{(\frac{1}{2})^i\exp(\frac{1}{2}X) X^{i-j}\} \\
    &= & (\frac{1}{2})^i\int_0^{\infty} \exp(\frac{1}{2}x)x^{i-j}\exp(-x)dx\\
&= &(\frac{1}{2})^{i-1}\int_0^{\infty}x^{i-j}\frac{1}{2}\exp(-\frac{1}{2}x)dx\\
&=&(\frac{1}{2})^{i-1}\frac{(i-j)!}{(\frac{1}{2})^{i-j}}
=(\frac{1}{2})^{j-1}(i-j)!. 
\end{array}
\end{equation}
Thus, by applying \eqref{eq:simpleboundsimp}, we conclude that
\begin{equation}
\begin{array}{@{}l @{\;} l @{\;} l}
    \mathcal{JG}(f,X) &\leq &\sum_{i=1}^{2k-1}\sum_{j=0}^i\frac{(-1)^{i-j+1}}{j!(i-j)!}\E\{f^{(i)}(X)X^{i-j}\}(\E X)^j \\ 
    &=&
\sum_{i=1}^{2k-1}\sum_{j=0}^i\frac{(-1)^{i-j+1}}{j!(i-j)!}(\frac{1}{2})^{j-1}(i-j)!1^j\\
&=&\sum_{i=1}^{2k-1}\sum_{j=0}^i\frac{1}{2^{j-1}}\frac{(-1)^{i-j+1}}{j!}
\end{array}
\end{equation}
and
\begin{equation}
\begin{array}{@{}l @{\;} l @{\;} l}
\mathcal{JG}(f,X) &\geq &    \sum_{i=1}^{2k-1}\sum_{j=0}^i\frac{(-1)^{i-j}}{j!(i-j)!}f^{(i)}(\E X)\E X^j(\E X)^{i-j} \\ 
 &=&
 \sum_{i=1}^{2k-1}\sum_{j=0}^i\frac{(-1)^{i-j}}{j!(i-j)!}(\frac{1}{2})^i\exp(\frac{1}{2})\frac{j!}{1^j}1^{i-j}\\
 &=&\sqrt{e}\sum_{i=1}^{2k-1} \frac{1}{2^i} \sum_{j=0}^i\frac{(-1)^{i-j}}{(i-j)!}. 
\end{array}
\end{equation}

\begin{figure}[t!]
    \centering
    \includegraphics[width=\columnwidth]{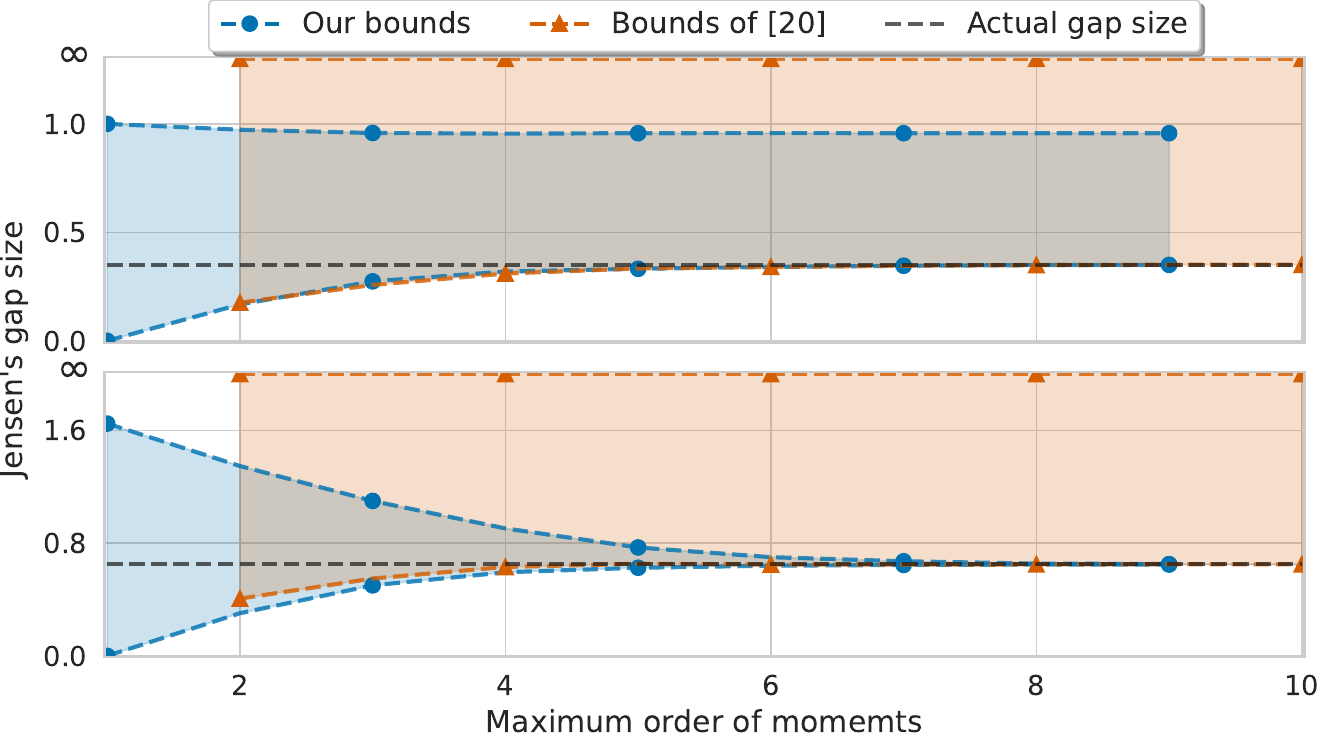}
    \caption{Lower and upper bounds on Jensen's gap given by our method and that of \citep{lee2021further} for $f(x)=\exp(\frac{1}{2}x)$ and $X\sim \text{Exponential}(1)$ (top), and for $f(x)=\exp(x)$ and $X\sim \text{Normal}(0,1)$ (bottom), vs. the maximum order of moments used in the calculations (i.e., $2k-1$ in the case of our method). 
    The dashed lines between the points represent second-degree polynomial interpolation. It should be noted that the upper bounds provided in \citep{lee2021further} are infinite.}
    \label{fig:liaolee1}
\end{figure}

In the upper part of Figure~\ref{fig:liaolee1} we compare the bounds given by \eqref{eq:jensensgapexpest} with those obtained in \citep{lee2021further}. (We emphasize that the method of \citep{liao2019sharpening} coincides with the method of \citep{lee2021further} of the first order.) 
It is clear that as the maximum number of moments used in the computations increases, our approach becomes superior (or at least comparable) while avoiding infinity. Note, however, that in \citep{lee2021further} the bounds are expressed using even order moments (see Theorem 2.1 in \citep{lee2021further}), whereas we use odd order moments. Consequently, although Figure~\ref{fig:liaolee1} suggests a comparison with missing moments interpolated by a second-degree polynomial, our method and that of \citep{lee2021further} should be considered as complementary rather than competitive in this case.

\paragraph{Bounds for Normally Distributed Random Variable}

We perform analogous calculations for $f(x)=\exp(x)$ and Gaussian (continuous) random variable $X\sim \text{Normal}(0,1)$. In this case, the exact Jensen's gap and the corresponding bounds are as follows:
\begin{equation}\label{eq:jensensgapexpnormal}
\begin{array}{@{}l @{\;} l @{\;} l}
\mathcal{JG}(f,X) & = & \sqrt{e}-1\approx 0.649
\end{array}
\end{equation}
and
\begin{equation}\label{eq:jensensgapexpnormalest}
\begin{array}{@{}l @{\;} l @{\;} l}
\sum_{i=1}^{2k-1}\frac{\E X^i}{i!}
&\leq & \mathcal{JG}(f,X)
 \leq  \sqrt{e} \sum_{i=1}^{2k-1}\frac{(-1)^{i+1}}{i!}\E Y^i,
\end{array}
\end{equation}
where $Y\sim \text{Normal}(1,1)$.
Indeed, note that
\begin{equation}
\begin{array}{@{}l @{\;} l @{\;} l}
    \E\{f^{(i)}(X)&\!X&\!\!^{i-j}\}=\E\{\exp(X)X^{i-j}\}\\
    &=&\int_{-\infty}^{\infty} \exp(x)x^{i-j}\frac{1}{\sqrt{2\pi}}\exp(-\frac{1}{2}x^2)dx\\
&=&\sqrt{e}\int_{-\infty}^{\infty} x^{i-j}\frac{1}{\sqrt{2\pi}}\exp\{-\frac{1}{2}(x-1)^2\}\\
&=&\sqrt{e}\, \E Y^{i-j},
\end{array}
\end{equation}
and therefore, since $\E X=0$, we can use \eqref{eq:simpleboundsimp2} to calculate the respective bounds in the following way:
\begin{equation}
\begin{array}{@{}l @{\;} l @{\;} l}
    \sum_{i=1}^{2k-1}\frac{1}{i!}f^{(i)}(0)&\!\E& \!\!\!X^i = \sum_{i=1}^{2k-1}\frac{\E X^i}{i!} \leq \mathcal{JG}(f,X) \\ 
    &\leq & \sqrt{e} \sum_{i=1}^{2k-1}\frac{(-1)^{i+1}}{i!}\E Y^i\\
    &=&-\sum_{i=1}^{2k-1}\frac{(-1)^{i}}{i!}\E\{f^{(i)}(X)X^{i}\}.
\end{array}
\end{equation}

The comparison of the bounds given in \eqref{eq:jensensgapexpnormalest} with those of \citep{lee2021further} is shown in the lower part of  Figure~\ref{fig:liaolee1}. Unlike the competitor, our approach provides tighter and tighter upper bounds as the maximum number of moments used in the computations increases. (Note that in our method there is no need to apply the support partitioning technique.) 

\subsection{Higher Order Bounds for Model Averaging and PAC-Bayes}

In this subsection, we relate our results to those presented in \citep{NEURIPS2020_3ac48664}. Since Theorem~\ref{thm:improvedlog} provides explicit, moment-based upper and lower bounds on Jensen's gap for the negative logarithmic function with a controllable expansion order via the parameter $k$, it allows for the systematic tightening of variational and PAC-Bayes bounds for cross-entropy losses, particularly in the context of probabilistic models.

To formalize this setup, let the training dataset be denoted by $\mathcal{D}=\{x_1,\ldots,x_n\}$, where each $x_i$ belongs to the data space $\mathcal{X}$. The conditional data model is given by a distribution $p(x|\theta)$ parameterized by $\theta$. We assume that the training data are sampled independently from a random variable $X$ with an unknown data-generating distribution $\nu$. We define the posterior predictive distribution induced by a probability distribution $\varrho$ over the parameters $\theta$ as 
$p(x) = \mathbb{E}_{\theta \sim \varrho}[p(x|\theta)]$.
If $\varrho$ is the Bayesian posterior, this simply reduces to Bayesian model averaging.

Our key quantity of interest is the generalization risk (cross-entropy), i.e,
\begin{equation}
\begin{array}{@{}l @{\;} l @{\;} l}
\textit{CE}(\varrho) &= &\mathbb{E}_{X \sim \nu}\{-\log \mathbb{E}_{\theta \sim \varrho}[p(X|\theta)]\}.
\end{array}
\end{equation}
We aim to identify the optimal probability distribution $\varrho^*$ over $\theta$ for model averaging, with the goal of achieving the best generalization performance. Specifically, we seek the distribution $\varrho^*$ that minimizes the cross-entropy loss. Formally, this is expressed as follows:
\begin{equation}
\begin{array}{@{}l @{\;} l @{\;} l}
\varrho^* &=& \arg\min_{\varrho} \textit{CE}(\varrho).
\end{array}
\end{equation}
Since the true data-generating distribution \(\nu\) is unknown, we aim to approximate the optimal posterior \(\varrho^*\) by minimizing a PAC-Bayes upper bound on the cross-entropy loss $\textit{CE}(\varrho)$, which depends on the observed dataset \(\mathcal{D}\). While traditional variational and PAC-Bayesian approaches bound this expression by the expected log-loss via Jensen's inequality, which only provides a first order bound, this can be very loose.

To overcome this limitation, we propose to apply Theorem~\ref{thm:improvedlog}, which leads to a series of increasingly tighter approximations to the cross-entropy loss. Note that our formulation is somewhat related to the second order bound explored in \citep{NEURIPS2020_3ac48664} (see Theorem 2 therein), but significantly extends this idea by allowing arbitrary higher order corrections, offering a more flexible and accurate approximation framework (see Figure~\ref{fig:extended_jensen_bounds}).

\begin{thm}
\label{thm:JensenkGAP}
For an arbitrary distribution $\varrho$ over $\theta$ and for any positive integer $k$, the following inequality hold:
\begin{equation}
\!\!\begin{array}{@{}l @{\;} l @{\;} l}
\textit{CE}(&\!\!\varrho&\!\!\!\!\!)\leq  \mathbb{E}_{\theta \sim \varrho}[L(\theta)]\\ &- &\sum_{i=1}^{2k-1}\frac{(-1)^i}{i}  \mathbb{E}_{X \sim \nu}\left\{\frac{\mathbb{E}_{\theta \sim \varrho}[(p(X|\theta)-\mu(X))^i]}{\mu(X)^{i}}\right\},
\end{array}
\end{equation}
where $L(\theta) = \mathbb{E}_{X \sim \nu}[-\log p(X|\theta)]$ is the log-loss of the model and $\mu(x)=\mathbb{E}_{\theta \sim \varrho}[p(x|\theta)]$, provided that appropriate finite expected values exist.
\end{thm}
\begin{proof} It follows from the left inequality in \eqref{eq:improveboundslog1}.
\end{proof}

We emphasize that direct minimization of the oracle bound provided by Theorem~\ref{thm:JensenkGAP} is not feasible, as it relies on expectations with respect to the true data distribution $\nu$, which is unknown in practice. To overcome this limitation, we can shift to the PAC-Bayes framework, which enables the construction of data-dependent generalization bounds. Unlike the oracle bound, PAC-Bayes bounds can be empirically estimated and provide probabilistic guarantees, making them both practically and theoretically sound. In the PAC-Bayes approach, we first select a prior distribution $\pi$ over $\theta$, which represents our initial beliefs before any data has been observed. After observing the data, we consider a posterior distribution $\varrho$ over $\theta$, reflecting updated beliefs. The goal is to bound the expected generalization risk $\textit{CE}(\varrho)$ in terms of the empirical risk and a complexity term that penalizes deviation from the prior, typically measured by the Kullback-Leibler divergence $\textit{KL}(\varrho | \pi)$.

The classical PAC-Bayes claim states that with probability at least $1 - \delta$ over the draw of the dataset $\mathcal{D}$, the following inequality holds:
\begin{equation}
\begin{array}{@{}l @{\;} l @{\;} l}
\textit{CE}(\varrho) &\leq & \mathbb{E}_{\theta \sim \varrho}[\hat{L}(\theta, \mathcal{D})] + \text{complexity},
\end{array}
\end{equation}
where $\hat{L}(\theta, \mathcal{D})$ is the empirical log-loss and the complexity term typically includes $\textit{KL}(\varrho | \pi)$ scaled by the sample size $n$ and a $\log(1/\delta)$ confidence adjustment---for a concrete example, see Theorem 1 in \citep{NEURIPS2020_3ac48664}. However, such a first order upper bound can be quite loose, because the expected log-loss $\mathbb{E}_{\theta\sim \varrho}[L(\theta)]$ bounds the true cross-entropy risk $\textit{CE}(\varrho)$ via Jensen's inequality, ignoring higher order information. By applying Theorem~\ref{thm:JensenkGAP}, we can obtain refined, data-dependent PAC-Bayes bounds that correct for this slackness using moment-based corrections. For example, applying the theorem with $k=2$ (third order correction), we obtain the following inequality:
\begin{equation}
\begin{array}{@{}l @{\;} l @{\;} l}
\textit{CE}(\varrho) & \leq & \mathbb{E}_{\theta\sim \varrho}[\hat{L}(\theta, \mathcal{D})] - \frac{1}{2n}\sum_{i=1}^n \frac{\mathbb{E}_{\theta\sim \varrho}[(p(x_i|\theta)-\mu_i)^2]}{\mu_i^2} \\
& + & \frac{1}{3n}\sum_{i=1}^n  \frac{\mathbb{E}_{\theta\sim \varrho}[(p(x_i|\theta)-\mu_i)^3]}{\mu_i^3} + \text{complexity},
\end{array}
\end{equation}
where $\mu_i = \mathbb{E}_{\theta\sim \varrho}[p(x_i|\theta)]$.
Although we do not explicitly formulate the resulting PAC-Bayes bounds in this work, such corrections are theoretically compatible with the PAC-Bayes framework and are a promising direction for future development.

The practical impact of the higher order corrections from Theorem~\ref{thm:JensenkGAP} is illustrated in Figure~\ref{fig:extended_jensen_bounds}, which extends the binomial model-averaging experiment proposed in \citep{NEURIPS2020_3ac48664} to include third/fifth-order bounds. Two scenarios are analyzed: \textit{misspecification} (no model $p(\cdot|\theta)$ matches true $\nu$) and \textit{perfect specification} (there exists $\theta^*$ with $p(\cdot|\theta^*)=\nu$). Models are mixed via $\varrho \in [0,1]$, where extreme values select single models and intermediate values weight averages. Under misspecification, cross-entropy loss minimizes at intermediate $\varrho^*$, supporting model averaging, while expected log-loss minimizes at extreme $\varrho^*=1$ via Dirac delta selection. Theorem~\ref{thm:JensenkGAP} enables progressively tighter bounds (third/fifth order), surpassing second-order bound from \citep{NEURIPS2020_3ac48664} in tracking true cross-entropy, achieving lower minima closer to $\varrho^*$. In perfect specification, cross-entropy loss and all bounds minimize at $\varrho^*=1$, confirming true model selection as optimal. Higher-order bounds align precisely with expected log-loss here, validating their consistency in well-specified scenarios. This illustrates how increasing bound orders improve approximation accuracy across both experimental settings, with practical implications for model selection and averaging strategies.

\subsection{Initial Comparison With the Results of \citep{pmlr-v206-struski23a}}\label{subsec:syntheticdataexperiments}

In this subsection, we focus on random variables with a gamma or log-normal distribution. Our goal is to compare the bounds computed by \eqref{eq:improveadditivedboundslog} with those presented in Corollary 4 of \citep{pmlr-v206-struski23a}, i.e.,
\begin{equation}\label{eq:improveboundslogstruski}
\begin{array}{@{}l @{\;} l @{\;} l}
0\leq \log(\E X)-\E\{\log(X)\}
&  \leq & \log (\E \frac{Y}{X}),
\end{array}
\end{equation}
where $Y$ is an independent copy of $X$. (Note that for such random variables and distributions, comparison with the other considered approaches is meaningless, since the lower bound given by Theorem 2.1 of \citep{lee2021further}, which is a generalization of Theorem 1 of \citep{liao2019sharpening}, coincides with that given by our method, while the upper bound is infinite.) In Figure  \ref{fig:theoretical_lognormal_synthetic}, we present the results of direct bounds computations using 
 \eqref{eq:simpleboundsimpgamma} and \eqref{eq:improvedboundslog_for_lognormal}. Note that for both situations considered, there are wide ranges of distribution parameters (depending on $k$) for which our method (compared to that of \citep{pmlr-v206-struski23a}) provides better bounds. (In fact, in \citep{pmlr-v206-struski23a} only a non-trivial upper bound for Jensen's gap in provided.) This motivates our experiments on log-likelihood estimation on real-world data, which are presented in the following subsections. 

\subsection{Log-Likelihood Estimation}\label{subsec:loglikelihoodestimations}

In this subsection, we present a novel method of log-likelihood estimation for arbitrary probabilistic models (with the latent).
Although our approach could be useful for different architectures, motivated by the work of \citep{pmlr-v206-struski23a}, we focus our attention on those with an autoencoder structure (such as VAEs). 

Let $p(x,z)$ be a joint model distribution on $\mathcal{X}\times \mathcal{Z}$, where $\mathcal{X}$ and $\mathcal{Z}$ denote data and latent spaces, respectively. 
In the case of an autoencoder-based model, where we also have a random encoder $q(z|x)$, a random decoder $p(x|z)$, and a prior $p(z)$, a model log-likelihood of a given data point $x\in \mathcal{X}$ is expressed as follows: 
\begin{equation}\label{eq:modelloglikelihood}
\begin{array}{@{}l @{\;} l @{\;} l}
\log [p(x)] & = & \int p(x,z) dz
=\log [\E_{Z \sim q(\cdot |x)} \frac{p(x,Z)}{q(Z|x)}]\\ 
&=& \log [\E_{Z \sim q(\cdot |x)} \frac{p(x|Z)p(Z)}{q(Z|x)}].
\end{array}
\end{equation}
Then, applying \eqref{eq:improveadditivedboundslogmeans} to the random variable $R_x(Z)=\frac{p(x|Z)p(Z)}{q(Z|x)}$ (for $Z\sim q(\cdot |x)$), we can compute rigorous lower and upper bounds on the log-likelihood of the model distribution in $x$ (note that this method requires the use of a number of independent copies of $R_x(Z)$, and hence $Z$).

Finally, we need to explore the conditions under which the bounds might be tight enough. To do this, we use the method described in the last paragraphs of Subsection \ref{subsec:boundslog}. Specifically, we approximate the distribution of the $n$-sample mean by $\text{Lognormal}(\mu,\sigma)$ with a small variance (hence a small $\sigma$), which requires taking $n$ sufficiently large. However, as the preceding analysis shows (see Subsection \ref{subsec:syntheticdataexperiments}), our method benefits more when $k$ is relatively small and $\sigma$ takes moderately small rather than very small values. Therefore, also considering that the computational complexity grow with increasing $k$ and $n$, some kind of trade-off would be desirable in this case. In summary, we postulate that computing lower and upper bounds for a model log-likelihood of a given data point $x\in \mathcal{X}$ using \eqref{eq:improveadditivedboundslogmeans} (as described above) should be based on the following criteria: (i) good log-normal approximation (see \citep{mouri2013log})---$n$ large enough, (ii) tight bounds---sufficiently small $\sigma$ (due to large $n$), (iii) clear advantage over existing methods---relatively small $k$ and balanced $\sigma$ (due to balanced $n$), and (iv) reasonable computational complexity---possibly small $k$ and $n$. With respect to the results presented in the previous subsection, at this stage, we can set reasonable values for $k$ as 2 or 3 (see Figure \ref{fig:theoretical_lognormal_synthetic}). 
On the other hand, determining an appropriate $n$ requires further analysis of the (data-dependent) distribution of the random variable $R_x(Z)$ (see the following subsection).

\begin{figure}[t!]
  \centering
  \includegraphics[width=\columnwidth]{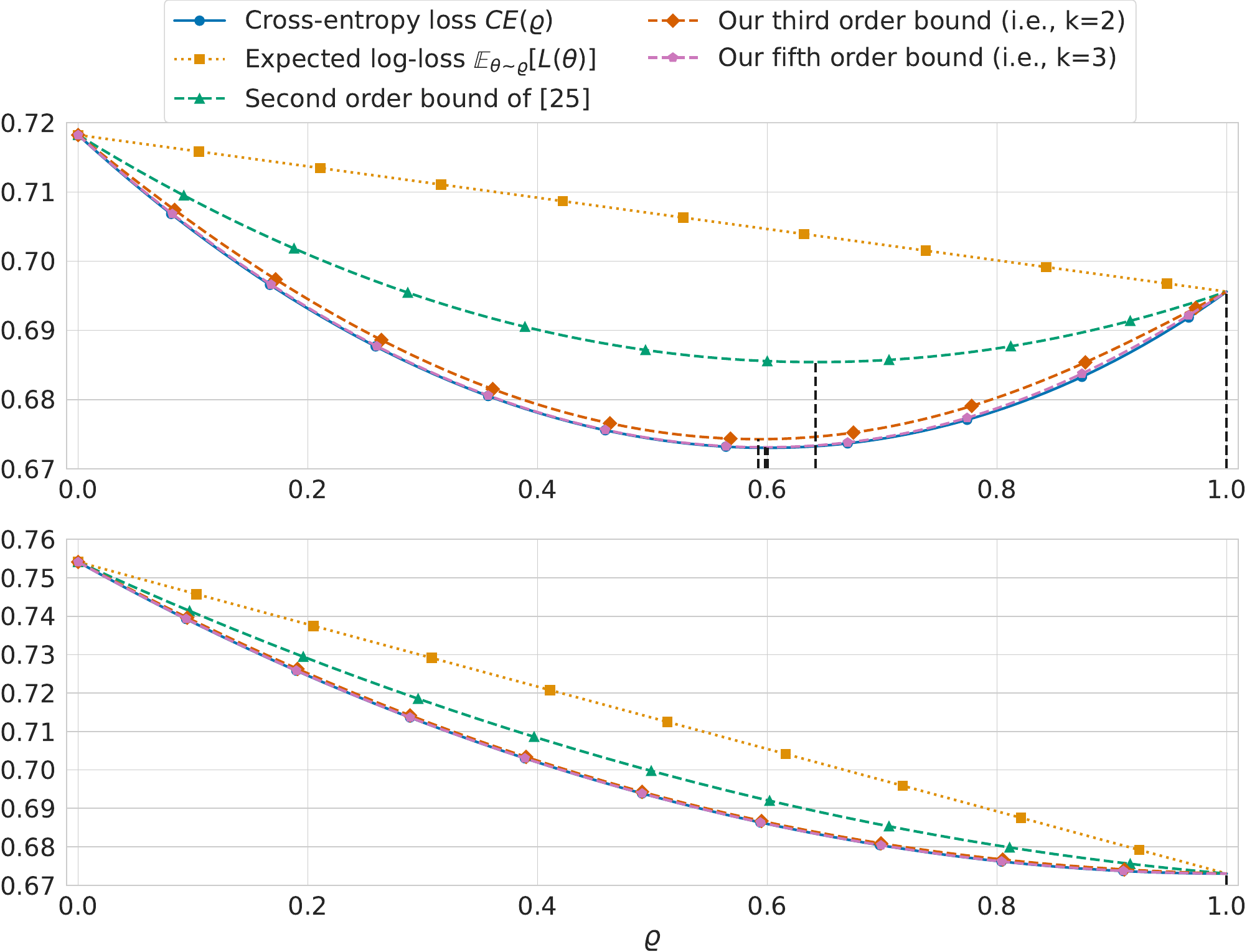}
  \caption{Upper bounds on the cross-entropy loss for the model misspecification setting (top) and the perfect model setting (bottom). The experimental setup from \citep{NEURIPS2020_3ac48664} was applied (see Figure 2 and Appendix B therein). Note that higher order bounds (ours) systematically improve the tightness and more accurately reflect the true minimum.}
  \label{fig:extended_jensen_bounds}
\end{figure}

\begin{figure}[t!]
   \centering
   \includegraphics[width=\columnwidth]{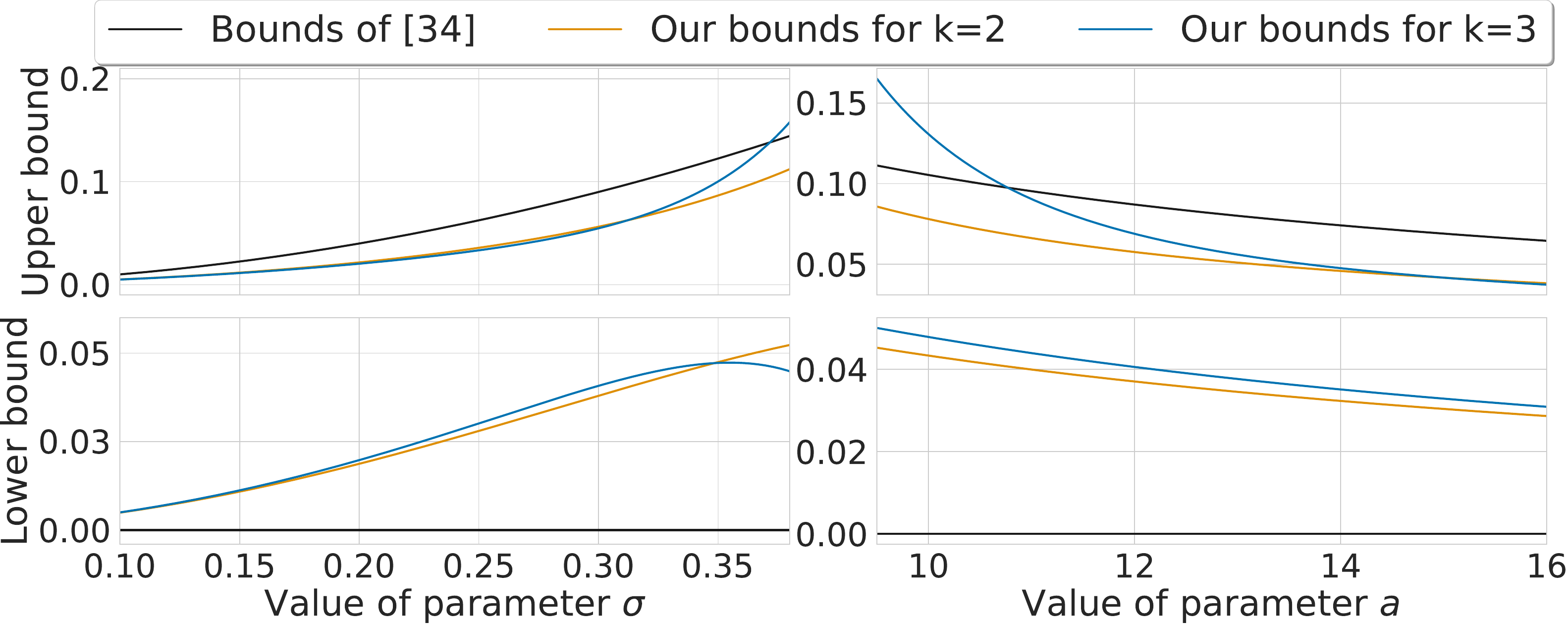}
      \caption{Lower and upper bounds on Jensen's gap for $f=-\log$ and $X\sim \text{Lognormal}(\mu,\sigma)$ (left) or $X\sim \text{Gamma}(a,\theta)$ (right), vs. different values of the respective distribution parameters, rigorously computed using our method and the method of \citep{pmlr-v206-struski23a}.}\label{fig:theoretical_lognormal_synthetic}
\end{figure}

\subsection{Experiments on Real-World Data}\label{subsec:experiments}

In this subsection, we present the results of experiments conducted to compare our log-likelihood estimation method with the state-of-the-art, which is represented by the recent experimental benchmark of \citep{pmlr-v206-struski23a} with VAE \citep{kingma2013auto,rezende2015variational}, IWAE-5, and IWAE-10 \citep{burda2015importance} models pre-trained on the MNIST~\citep{lecun1998gradient}, SVHN~\citep{netzer2011reading}, and CelebA~\citep{liu2015deep} datasets. In addition, we provide an ablation study that illustrates the impact of the log-likelihood distribution on the applicability of our approach.

Following the conclusions of the previous subsection, to compute our proposed lower and upper bounds by applying \eqref{eq:improveadditivedboundslogmeans} to the random variable $R_x(Z)$ that depends on a given data point $x$, we first try to satisfy postulates (i)--(iv) by selecting appropriate $n$ via an appropriate grid search procedure. 
Then (taking $k\in \{2,3\}$) we compute from \eqref{eq:improveadditivedboundslogmeans} the lower and upper bounds for $\log \E_{Z\sim q(\cdot|x)}R_x(Z)$.  Note that the upper bound on the respective Jensen's gap 
can be considered as a measure of the tightness of the estimation, so it is further reported. (In fact, we report the value computed using the sample mean estimator for each expectation, averaged over all data points $x$.)

\begin{figure}[t!]
\centering
\includegraphics[width=\columnwidth]
{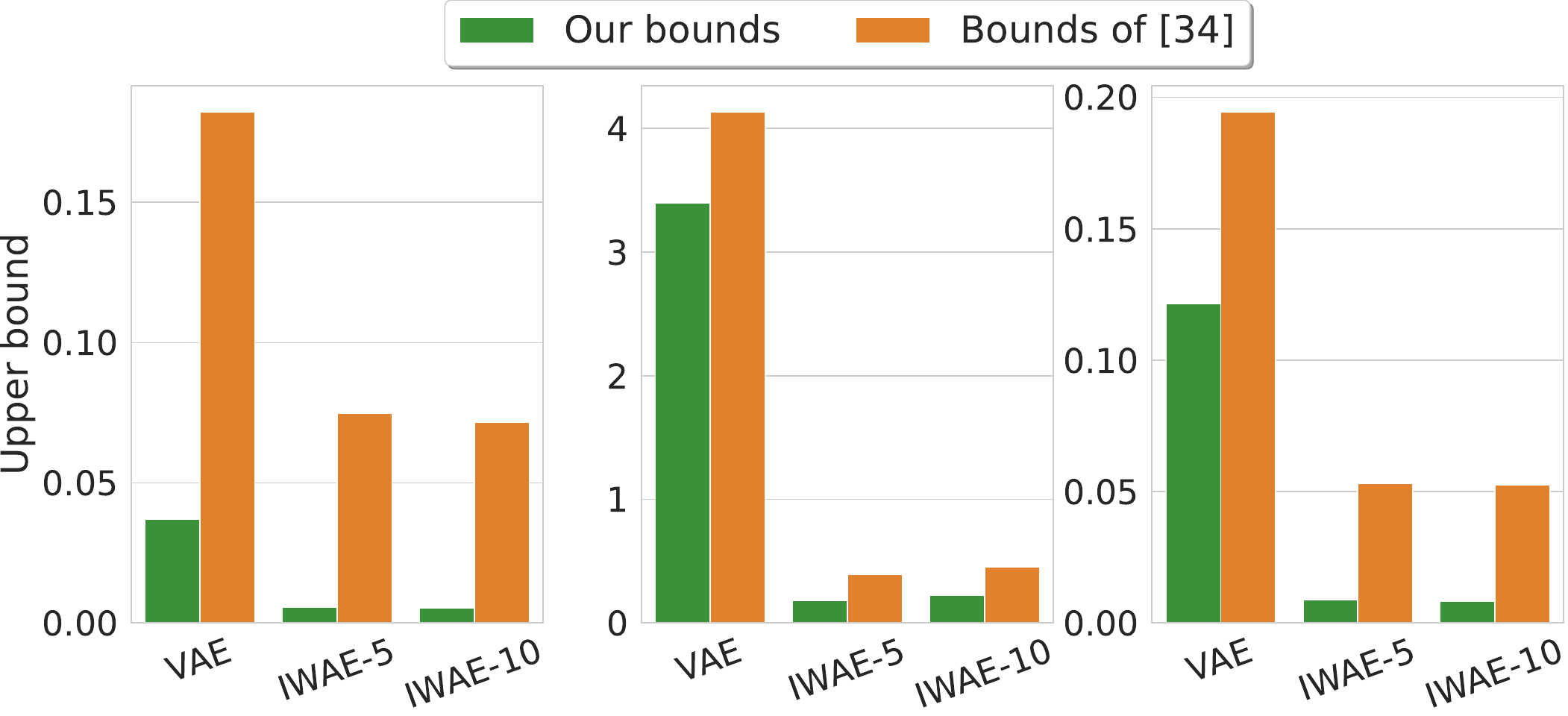}
\caption{Estimated upper bounds on Jensen's gap, which measure the tightness of the estimation of the log-likelihood of VAE, IWAE-5, and IWAE-10 models pre-trained on the MNIST (left), SVHN (middle), and CelebA (right) datasets (lower is better). 
All values shown are restricted to (and averaged over) all data points counted in Table \ref{tab:my_label}.}
\label{tab.gap}
\end{figure}

\begin{table}[t!]\small
    \caption{The number and percentage of data points (images) for which our method is better than the method of \citep{pmlr-v206-struski23a}. Note that incorporating our approach is beneficial for IWAEs trained on MNIST or SVHN datasets.}
    \label{tab:my_label}
    \centering
    \begin{tabular}{c@{\;\;\;\;}c@{\;\;\;\;}c@{\;\;\;\;}c}
    \toprule
      & \textbf{VAE} & \textbf{IWAE-5} & \textbf{IWAE-10}\\
    \midrule
    \textbf{MNIST} & 9593 (96.37\%)   & 9021 (100\%) & 8490 (100\%)\\
    \midrule
    \textbf{SVHN} & 59 (0.59\%)    & 4578 (46.07\%)   & 3653 (36.88\%)  \\
    \midrule
    \textbf{CelebA} & 18  (0.18\%)   & 99  (0.99\%)   & 76 (0.76\%) \\
    \bottomrule
    \end{tabular}
\end{table}

The results of our experiments are presented in Table \ref{tab:my_label} and Figure \ref{tab.gap}. Specifically, in Table \ref{tab:my_label} we provide the number and percentage of data points for which our method is superior to that of \citep{pmlr-v206-struski23a}. (We excluded points for which we could not obtain reasonable results due to some numerical problems.) It is clear that incorporating our approach is beneficial for models and datasets that produce less complicated latent distributions (such as IWAEs trained on MNIST or SVHN), while otherwise yielding little progress (we provide an additional ablation study on this phenomenon). Additionally, in Figure \ref{tab.gap} we compare the values computed by our method and the method of \citep{pmlr-v206-struski23a}, restricted to (and averaged over) all data points that benefited from the use of our method. It is evident that our approach provides superior bounds in these cases, probably due to the compliance with the proposed criteria (see the following paragraph).

 \paragraph{Ablation Study} The overarching observation from our experiments is the significant influence of the distribution of the random variable $\overline{R_x(Z)}$ on the effectiveness of our method of log-likelihood estimation over the state-of-the-art. From the results presented in Table \ref{tab:my_label}, we learn that for more complicated latent distributions with a significant number of outliers (e.g., those induced by the models pre-trained on the CelebA dataset), there are only a few data points for which we obtain superior results compared to the method of \citep{pmlr-v206-struski23a}. We suspect that this is because the log-normal approximation obtained is not good enough (see criterion (i)), although we have tried to find a sufficiently large $n$ in each case.

To substantiate the above hypothesis, we present a detailed breakdown of the results for random data points in Figure~\ref{fig:hist_qq-plot}. For each dataset and model combination, we have specifically selected two image examples: one illustrating a scenario where our method is superior to that of~\citep{pmlr-v206-struski23a}, and the other illustrating a situation where our bounds lag behind.
It is clear that in the first case we were able to achieve a Gaussian approximation for the distribution of $\log \overline{R_x(Z)}$, while in the other case we definitely failed.

\paragraph{Experimental Details} In all experiments, we used the experimental setup of \citep{pmlr-v206-struski23a}. Specifically, we examined variational autoencoders (VAEs) and importance-weighted autoencoders (IWAEs) pre-trained on the MNIST, SVHN, and CelebA datasets with a variety of objective functions, including the evidence lower bound (ELBO), importance-weighted ELBO with 5 importance samples ($\text{IW-ELBO}_5$), and importance-weighted ELBO with 10 importance samples ($\text{IW-ELBO}_{10}$). Notably, we kept the same neural architectures for the IWAE models as for the VAE framework.

In our experimental study across all datasets, we selected for evaluation 10000 random images from each test set. In particular, in the case of the MNIST dataset, we included the entire test set since it contains exactly 10000 images. For each of these individual images, we applied a grid search procedure, to find the smallest $n$ from the set of integers between 50 and 5000000 (depending on the dataset) for which the empirical standard deviation of $\log [\overline{R_x(z)}]=\log [\frac{1}{n}\sum_{i=1}^nR_x(z_i)]$ (where $x$ is an image from the dataset and $z$ is an $n$-sample from the latent space of a given generative model) does not exceed 0.3.

After determining $n$, we computed our upper bound for Jensen's gap, together with the one provided in \citep{pmlr-v206-struski23a}. All expectations arising in the respective bound formulas were estimated via averages over 10000 samples. 

\begin{figure}[t!]
    \centering
    \includegraphics[width=\columnwidth]{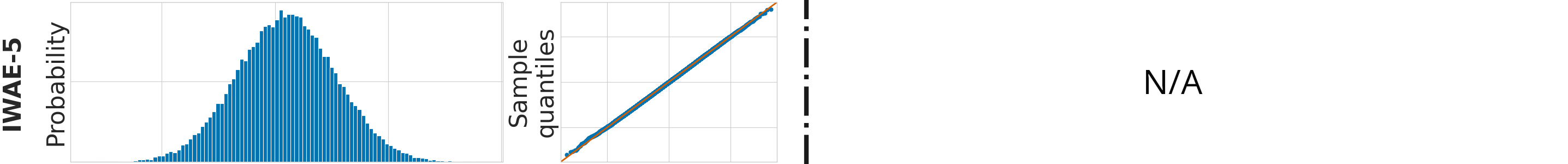}\\
        \includegraphics[width=\columnwidth]{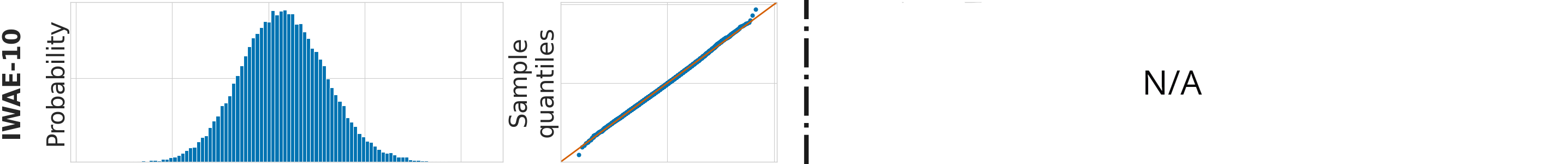}\\
        \includegraphics[width=\columnwidth]{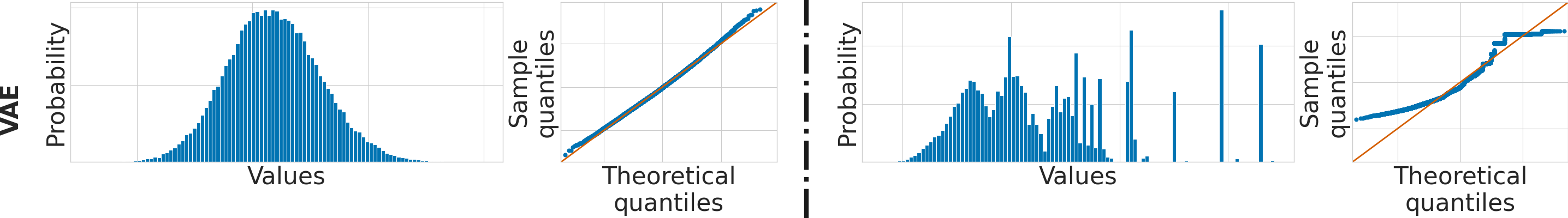}\\
        MNIST dataset : IWAE-5 (top), IWAE-10 (middle), VAE (bottom)
        \vskip1mm 
      \includegraphics[width=\columnwidth]{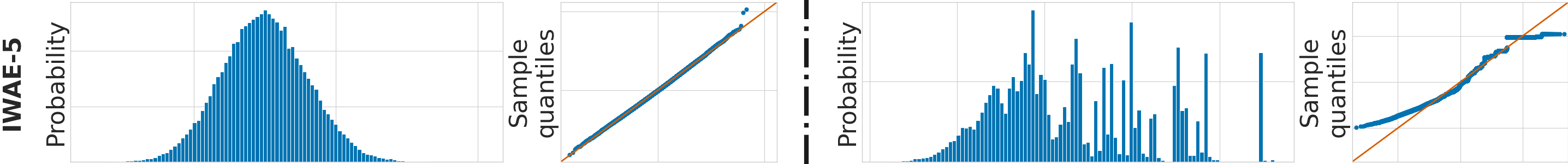}
        \includegraphics[width=\columnwidth]{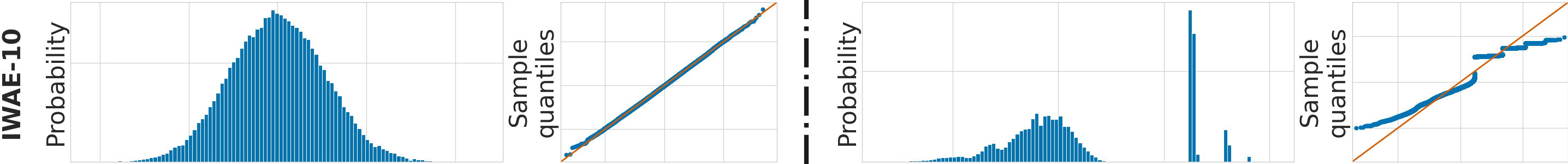}\\
        \includegraphics[width=\columnwidth]{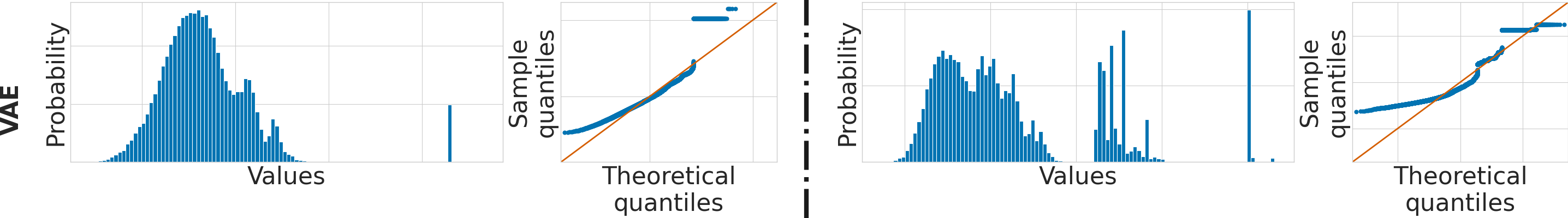}\\
        SVHN dataset: IWAE-5 (top), IWAE-10 (middle), VAE (bottom)
    \vskip1mm 
\includegraphics[width=\columnwidth]{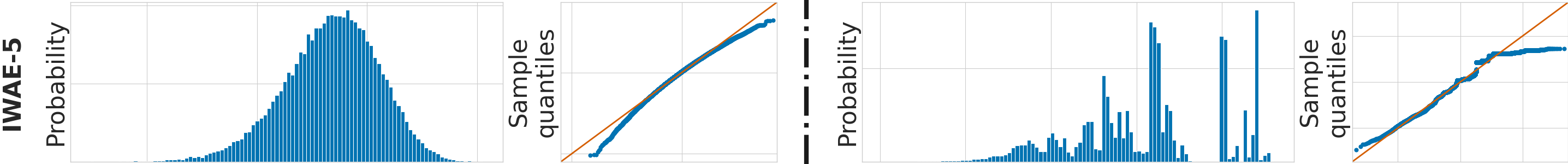}\\
        \includegraphics[width=\columnwidth]{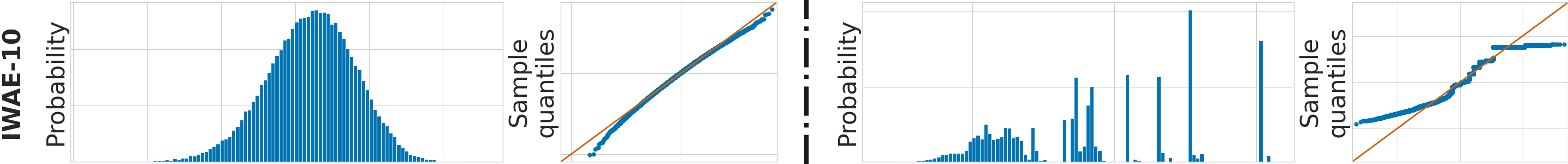}\\
        \includegraphics[width=\columnwidth]{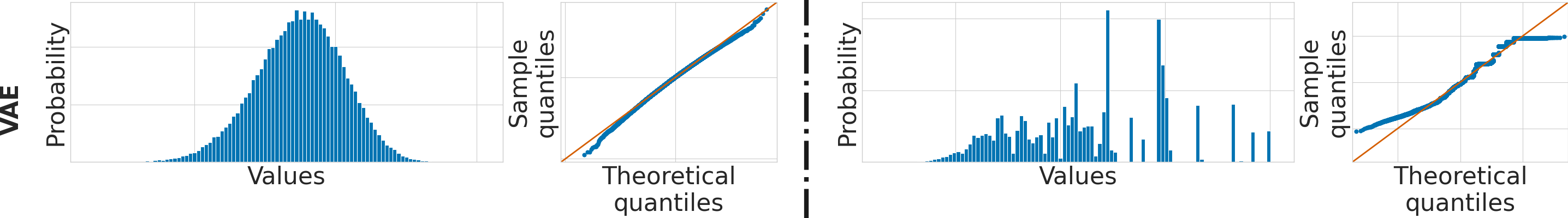}\\
        CelebA dataset: IWAE-5 (top), IWAE-10 (middle), VAE (bottom)
       \caption{Histograms and Q-Q plots~\citep{gnanadesikan1968probability,field2009shapiro} for randomly selected images associated with each dataset and model combination. The horizontal dashed line is a boundary between positive outcomes, where our method provides superior upper bound estimates, and negative outcomes, where our results are suboptimal. Notably, when our method wins, the underlying distribution closely approximates a Gaussian distribution, as shown in the Q-Q plot. Conversely, when our bound estimates are worse, the distribution deviates significantly from a Gaussian. Note that for the MNIST dataset and the IWAE-5/10 models, our method outperforms the state-of-the-art for all images, indicating that no images fall on the right side of the dashed horizontal line (marked N/A).}
    \label{fig:hist_qq-plot}
\end{figure}

\paragraph{Architecture Details}  
We used the same convolutional VAE architecture as in \citep{pmlr-v206-struski23a}. The model weights were optimized using the Adam optimizer with a learning rate of 0.0001. Training was performed for 100 epochs with batch size 64 on the CelebA dataset, and for 50 epochs on MNIST and SVHN datasets. We used Euclidean latent spaces $\mathcal{Z} = \mathbb{R}^d$, with dimensions $d = 8, 32, 128$ for MNIST, SVHN, and CelebA, respectively, and standard Gaussian priors $p(z) \sim \mathcal{N}(0, I_d)$.
The encoders used Gaussian coding,  characterized by $q(z|x) \sim \mathcal{N}(\mu_x, \Sigma_x)$, where $\Sigma_x$ is a diagonal covariance matrix. The decoders were Gaussian, with the following form: $p(x|z) \sim \mathcal{N}(\mu_z, \sigma^2 I)$, where $\sigma^2 = 0.3$.

For MNIST, the encoder consisted of two fully-connected layers, and the decoder of three fully-connected layers, with ReLU activations between layers. The SVHN architecture was deeper, with four layers each in encoder and decoder; the encoder used only 2D convolutions with leaky ReLU activation (leak 0.2), while the decoder used transposed 2D convolutions with ReLU activations, and a sigmoid activation in the final layer. For CelebA, the networks consisted mainly of repeated five-layer blocks; each encoder block contained a 2D convolution, batch normalization~\citep{ioffe2015batch}, and leaky ReLU (leak 0.2), while decoder blocks began with 2D nearest-neighbor upsampling, followed by a 2D convolution, batch normalization, and leaky ReLU activation (leak 0.2).

\section{Conclusions}
In this paper, we introduced a novel general lower and upper bound for Jensen's gap. By considering several special cases with different assumptions on the underlying function and distribution, we provided analytical and empirical arguments that our contribution has the potential to improve on state-of-the-art solutions provided in \citep{liao2019sharpening,lee2021further,pmlr-v206-struski23a}. In particular, we conducted experiments on real-world data, which demonstrated that our approach is superior to that of \citep{pmlr-v206-struski23a} as an efficient method for estimating the log-likelihood of probabilistic models, provided that the appropriate criteria are satisfied. Additionally, we demonstrated that our approach extends the second-order Jensen's gap approximation of \citep{NEURIPS2020_3ac48664} by providing a general framework for constructing higher-order bounds for the negative logarithmic function, leading to systematically tighter approximations of the generalization risk (cross-entropy). Finally, we linked these bounds to the PAC-Bayes approach, providing new insights into generalization
performance in probabilistic models.

\paragraph{Limitations}
Our contribution has some limitations that could be considered as starting points for future work. First, the superiority of the provided log-likelihood estimation method strongly depends on the properties of the underlying distribution, which may cause some problems in experiments on large-scale datasets.  Consequently, even though we provide rigorous bounds in all cases considered, expressed in terms of the moments of the respective random variables, their computation is often problematic in our experiments on real-world data, as they are based on biased estimators. Second, the considered real-world applications of this method are limited to probabilistic models, which do not exhaust all possible uses of variational inference in machine learning.  Finally, while our higher-order corrections are theoretically compatible with the PAC-Bayes framework, we do not provide explicit PAC-Bayes bounds in this work, which limits direct comparisons with existing PAC-Bayesian results (however, we plan to investigate this issue in future work).

\paragraph{Societal Impact}
We do not foresee any negative societal consequences of our contribution. However, even though the training of novel deep generative architectures was not the scope of our work, we cannot exclude the potential applicability of our approach for such purposes in further studies. Therefore, we must emphasize that the use of generative modeling in real-world applications requires careful monitoring to avoid amplifying societal biases (which are present in the data).

\begin{acks}
This research was partially funded by the National Science Centre, Poland, grants no. 2023/50/E/ST6/00068 (work of Marcin Mazur) and 2023/49/B/ST6/01137 (work of Łukasz Struski). Some experiments were performed on servers purchased with funds from the flagship project entitled ``Artificial Intelligence Computing Center Core Facility'' from the DigiWorld Priority Research Area within the Excellence Initiative -- Research University program at Jagiellonian University in Kraków.
\end{acks}

\section*{GenAI Usage Disclosure}
Generative AI software tools were used exclusively during the writing stage to edit and improve the clarity and quality of the existing manuscript text. No AI-generated content was used to produce novel research ideas, analyses, or results.

\bibliographystyle{ACM-Reference-Format}
\balance
\bibliography{ref}

@article{banic2008superquadratic,
  title="Superquadratic functions and refinements of some classical inequalities",
  author="Banic, Senka and Pecaric, Josip and Varosanec, Sanja",
  journal="Journal of the Korean Mathematical Society",
  volume="45",
  pages="513--525",
  year="2008"
}

@article{lovrivcevic2018zipf,
  title="Zipf--Mandelbrot law, f-divergences and the Jensen-type interpolating inequalities",
  author="Lovričević, Neda and Pečarić, Dilda and Pečarić, Josip",
  journal="Journal of inequalities and applications",
  volume="2018",
  number="1",
  pages="36",
  year="2018",
  publisher="Springer"
}

@article{sandifer2006euler,
  title="How Euler did it",
  author="Sandifer, Edward",
  journal="Washington: Mathematics Association of America",
  year="2006",
  publisher="Citeseer"
}

@article{abramovich2022new,
  title="New inequalities related to superquadratic functions",
  author="Abramovich, Shoshana",
  journal="Aequationes mathematicae",
  volume="96",
  number="1",
  pages="201--219",
  year="2022",
  publisher="Springer"
}

@article{liao2019sharpening,
  title="Sharpening Jensen's inequality",
  author="Liao, JG and Berg, Arthur",
  journal="The American Statistician",
  volume="73",
  number="3",
  pages="278--281",
  year="2019",
  publisher="Taylor \& Francis"
}

@article{bakula2016converse,
  title="On the converse {Jensen} inequality for strongly convex functions",
  author="Bakula, Milica Klari{\v{c}}i{\'c} and Nikodem, Kazimierz",
  journal="Journal of Mathematical Analysis and Applications",
  volume="434",
  number="1",
  pages="516--522",
  year="2016",
  publisher="Elsevier"
}

@article{horvath2014refinement,
  title="Refinement of {Jensen’s} inequality for operator convex functions",
  author="Horv{\'a}th, L and Khan, KA and Pecaric, J",
  journal="Adv. Inequal. Appl.",
  volume="2014",
  pages="Article--ID",
  year="2014"
}

@article{lee2021further,
  title="Further sharpening of {Jensen's} inequality",
  author="Lee, Sang Kyu and Chang, Jae Ho and Kim, Hyoung-Moon",
  journal="Statistics",
  volume="55",
  number="5",
  pages="1154--1168",
  year="2021"
}

@article{gao2017bounds,
  title="Bounds on the {Jensen} gap, and implications for mean-concentrated distributions",
  author="Gao, Xiang and Sitharam, Meera and Roitberg, Adrian E",
  journal="arXiv preprint arXiv:1712.05267",
  year="2017"
}

@article{lecun1998gradient,
  title="Gradient-based learning applied to document recognition",
  author="LeCun, Yann and Bottou, L{\'e}on and Bengio, Yoshua and Haffner, Patrick",
  journal="Proceedings of the IEEE",
  volume="86",
  number="11",
  pages="2278--2324",
  year="1998",
  publisher="Ieee"
}

@article{mouri2013log,
  title="Log-normal distribution from a process that is not multiplicative but is additive",
  author="Mouri, Hideaki",
  journal="Physical Review E",
  volume="88",
  number="4",
  pages="042124",
  year="2013",
  publisher="APS"
}

@inproceedings{liu2015deep,
  title="Deep learning face attributes in the wild",
  author="Liu, Ziwei and Luo, Ping and Wang, Xiaogang and Tang, Xiaoou",
  booktitle="Proceedings of the IEEE International Conference on Computer Vision",
  pages="3730--3738",
  year="2015"
}

@article{netzer2011reading,
title="{SVHN:} Reading Digits in Natural Images with Unsupervised Feature Learning",
 author="Netzer, Yuval and Wang, Tao and Coates, Adam and Bissacco, Alessandro and Wu, Bo and Ng, Andrew Y",
 journal="NIPS Workshop on Deep Learning and Unsupervised Feature Learning",
year="2011"
}

@article{burda2015importance,
  title="Importance weighted autoencoders",
  author="Burda, Yuri and Grosse, Roger and Salakhutdinov, Ruslan",
  journal="arXiv preprint arXiv:1509.00519",
  year="2015"
}

@article{kingma2013auto,
  title="Auto-encoding variational {Bayes}",
  author="Kingma, Diederik P and Welling, Max",
  journal="arXiv preprint arXiv:1312.6114",
  year="2013"
}

@inproceedings{rezende2015variational,
  title="Variational inference with normalizing flows",
  author="Rezende, Danilo and Mohamed, Shakir",
  booktitle="International Conference on Machine Learning",
  pages="1530--1538",
  year="2015",
  organization="PMLR"
}

@inproceedings{nowozin2018debiasing,
  title="Debiasing evidence approximations: On importance-weighted autoencoders and jackknife variational inference",
  author="Nowozin, Sebastian",
  booktitle="International Conference on Learning Representations",
  year="2018"
}

@article{maddison2017filtering,
  title="Filtering variational objectives",
  author="Maddison, Chris J and Lawson, Dieterich and Tucker, George and Heess, Nicolas and Norouzi, Mohammad and Mnih, Andriy and Doucet, Arnaud and Teh, Yee Whye",
  journal="arXiv preprint arXiv:1705.09279",
  year="2017"
}

@inproceedings{ioffe2015batch,
  title="Batch normalization: Accelerating deep network training by reducing internal covariate shift",
  author="Ioffe, Sergey and Szegedy, Christian",
  booktitle="International Conference on Machine Learning",
  pages="448--456",
  year="2015",
  organization="PMLR"
}

@article{ji2019stochastic,
  title="Stochastic variational inference via upper bound",
  author="Ji, Chunlin and Shen, Haige",
  journal="arXiv preprint arXiv:1912.00650",
  year="2019"
}

@article{grosse2015sandwiching,
  title="Sandwiching the marginal likelihood using bidirectional {Monte Carlo}",
  author="Grosse, Roger B and Ghahramani, Zoubin and Adams, Ryan P",
  journal="arXiv preprint arXiv:1511.02543",
  year="2015"
}

@article{dieng2017variational,
  title="Variational Inference via $\chi$ Upper Bound Minimization",
  author="Dieng, Adji Bousso and Tran, Dustin and Ranganath, Rajesh and Paisley, John and Blei, David",
  journal="Advances in Neural Information Processing Systems",
  volume="30",
  year="2017"
}

@article{masrani2019thermodynamic,
  title="The thermodynamic variational objective",
  author="Masrani, Vaden and Le, Tuan Anh and Wood, Frank",
  journal="Advances in Neural Information Processing Systems",
  volume="32",
  year="2019"
}

@article{walker2014lower,
  title="On a lower bound for the {Jensen} inequality",
  author="Walker, Stephen G",
  journal="SIAM Journal on Mathematical Analysis",
  volume="46",
  number="5",
  pages="3151--3157",
  year="2014",
  publisher="SIAM"
}

@article{jensen1906fonctions,
  title="Sur les fonctions convexes et les in{\'e}galit{\'e}s entre les valeurs moyennes",
  author="Jensen, Johan Ludwig William Valdemar",
  journal="Acta mathematica",
  volume="30",
  number="1",
  pages="175--193",
  year="1906",
  publisher="Springer"
}

@article{dragomir2015inequality,
  title="Inequality for power series with nonnegative coefficients and applications",
  author="Dragomir, Silvestru Sever",
  journal="Open Mathematics",
  volume="13",
  number="1",
  year="2015",
  publisher="De Gruyter Open Access"
}

@article{abramovich2004refining,
  title="Refining {Jensen's} inequality",
  author="Abramovich, Shoshana and Jameson, Graham and Sinnamon, Gord",
  journal="Bulletin math{\'e}matique de la Soci{\'e}t{\'e} des Sciences Math{\'e}matiques de Roumanie",
  pages="3--14",
  year="2004",
  publisher="JSTOR"
}

@article{dragomir2001some,
  title="Some inequalities for (m,{M})-convex mappings and applications for the Csisz{\'a}r $\Phi$-divergence in information theory",
  author="Dragomir, Sever S",
  journal="Mathematical Journal of Ibaraki University",
  volume="33",
  pages="35--50",
  year="2001",
  publisher="Department of Mathematics, Faculty of Science, Ibaraki University"
}

@article{pecaric1985companion,
  title="A companion to {Jensen-Steffensen’s} inequality",
  author="Pecaric, JE",
  journal="J. Approx. Theory",
  volume="44",
  number="3",
  pages="289--291",
  year="1985"
}

@InProceedings{pmlr-v206-struski23a,
  title = 	 "Bounding Evidence and Estimating Log-Likelihood in {VAE}",
  author =       "Struski, {\L}ukasz and Mazur, Marcin and Batorski, Pawe{\l} and Spurek, Przemys{\l}aw and Tabor, Jacek",
  booktitle = 	 "Proceedings of The 26th International Conference on Artificial Intelligence and Statistics",
  pages = 	 "5036--5051",
  year = 	 "2023",
  editor = 	 "Ruiz, Francisco and Dy, Jennifer and van de Meent, Jan-Willem",
  volume = 	 "206",
  series = 	 "Proceedings of Machine Learning Research",
  publisher =    "PMLR"
}

@article{gnanadesikan1968probability,
  title="Probability plotting methods for the analysis of data",
  author="Gnanadesikan, Ramanathan and Wilk, Martin B",
  journal="Biometrika",
  volume="55",
  number="1",
  pages="1--17",
  year="1968"
}

@article{field2009shapiro,
  title="Discovering statistics using SPSS",
  author="Field Andy",
  journal="SAGE Publications",
  volume="",
  number="",
  pages="143",
  year="2009"
}

@inproceedings{NEURIPS2020_3ac48664,
 author = "Masegosa, Andres",
 booktitle = "Advances in Neural Information Processing Systems",
 editor = "H. Larochelle and M. Ranzato and R. Hadsell and M.F. Balcan and H. Lin",
 pages = "5479--5491",
 title = "Learning under Model Misspecification: Applications to Variational and Ensemble methods",
 volume = "33",
 year = "2020"
}

@inproceedings{mcallester1999pac,
  title={PAC-Bayesian model averaging},
  author={McAllester, David A},
  booktitle={Proceedings of the twelfth annual conference on Computational learning theory},
  pages={164--170},
  year={1999}
}

@article{catoni2007pac,
  title={PAC-Bayesian supervised classification: the thermodynamics of statistical learning},
  author={Catoni, Olivier},
  journal={arXiv preprint arXiv:0712.0248},
  year={2007}
}

\end{document}